\documentclass[10pt]{article}
\usepackage[utf8]{inputenc} %
\usepackage{amsmath}
\usepackage{amsfonts}
\usepackage{amssymb}
\usepackage{mathrsfs}                  
\usepackage{amsthm}
\usepackage{enumitem}
\usepackage{amscd}
\usepackage{xcolor}
\usepackage{geometry}
\usepackage{bbm}
\usepackage{bm}
\usepackage{subcaption}
\usepackage{graphicx}       %
\usepackage[utf8]{inputenc}
\usepackage{multirow}
\usepackage{latexsym,graphics,amscd} 
\usepackage{bigints}
\usepackage{booktabs} 
\usepackage[ruled,vlined]{algorithm2e}
\usepackage{tabularx} 
\usepackage{threeparttable}
\usepackage{lscape}
\usepackage{caption} 
\usepackage{mathtools}
\usepackage{array} 
\usepackage{makecell}
\usepackage{url}
\usepackage[final]{pdfpages}
\usepackage{hyperref}
\hypersetup{
    colorlinks=true,
    citecolor=blue,
    linkcolor=blue,
    filecolor=magenta,      
    urlcolor=blue,
}
\usepackage{fancyhdr}       %

\numberwithin{equation}{section}
\allowdisplaybreaks

\theoremstyle{plain}
\newtheorem{theorem}{Theorem}[section]

\newtheorem{lemma}[theorem]{Lemma}

\newtheorem{assumption}{Assumption}
\theoremstyle{definition}
\newtheorem{definition}[theorem]{Definition}

\theoremstyle{remark}
\newtheorem{remark}[theorem]{Remark}

\addtocounter{footnote}{1}
\makeatletter
\@addtoreset{table}{bsection}
\def\thetable{\thesection.\@arabic\c@table}
\def\fps@table{h, t}
\@addtoreset{equation}{section}

\makeatother

\providecommand{\abs}[1]{\lvert#1\rvert}
\providecommand{\norm}[1]{\lVert#1\rVert}
\newcommand{\bfi}{\bfseries\itshape}
\newcommand{\vertiii}[1]{{\left\vert\kern-0.25ex\left\vert\kern-0.25ex\left\vert #1 
    \right\vert\kern-0.25ex\right\vert\kern-0.25ex\right\vert}}
\newsavebox{\savepar}

\makeatletter
\makeatletter

\usepackage{scalerel,stackengine}
\stackMath
\newcommand\reallywidehat[1]{%
\savestack{\tmpbox}{\stretchto{%
  \scaleto{%
    \scalerel*[\widthof{\ensuremath{#1}}]{\kern-.6pt\bigwedge\kern-.6pt}%
    {\rule[-\textheight/2]{1ex}{\textheight}}%
  }{\textheight}%
}{0.5ex}}%
\stackon[1pt]{#1}{\tmpbox}%
}

\newcommand{\Prob}{{\mathbb{P}}}

\newcommand{\Real}{{\mathbb{R}}}

\newcommand{\R}{\mathbb{R}}

\newcommand{\Z}{\mathbb{Z}}
\newcommand{\N}{\mathbb{N}}

\let\svthefootnote\thefootnote

\begin{document}
\title{\textbf{%
    Memory capacity of nonlinear recurrent networks: Is it informative?
}}
\author{Giovanni Ballarin$^{1}$, Lyudmila Grigoryeva$^{1,2}$, Juan-Pablo Ortega$^{3}$}
\date{}
\maketitle

\begin{abstract}
    The total memory capacity (MC) of linear recurrent neural networks (RNNs) has been proven to be equal to the rank of the corresponding Kalman controllability matrix, and it is almost surely maximal for connectivity and input weight matrices drawn from regular distributions. This fact questions the usefulness of this metric in distinguishing the performance of linear RNNs in the processing of stochastic signals. This work shows that the MC of random nonlinear RNNs yields arbitrary values within established upper and lower bounds depending exclusively on the scale of the input process. This confirms that the existing definition of MC in linear and nonlinear cases has no practical value.
\end{abstract}

\bigskip

\noindent \textbf{Keywords:} memory capacity, nonlinear recurrent neural networks, reservoir computing, echo state networks.  \\

\makeatletter
\addtocounter{footnote}{1} \footnotetext{ 
Mathematics and Statistics Division, University of St. Gallen, Rosenbergstrasse 22, CH-9000 St.~Gallen, Switzerland. {\texttt{Giovanni.Ballarin@unisg.ch} }, {\texttt{Lyudmila.Grigoryeva@unisg.ch} }}
\addtocounter{footnote}{1} \footnotetext{%
Honorary Associate Professor, Department of Statistics, University of Warwick, Coventry CV4 7AL, UK. {\texttt{Lyudmila.Grigoryeva@warwick.ac.uk} }} 
\addtocounter{footnote}{1} \footnotetext{%
Division of Mathematical Sciences, School of Physical and Mathematical Sciences,
Nanyang Technological University,
21 Nanyang Link,
Singapore 637371.
{\texttt{Juan-Pablo.Ortega@ntu.edu.sg}}}
\let\thefootnote\relax\footnote{%
	~\\%
	Accepted version for the 
	7th International Conference on
	Geometric Science of Information, GSI 2025, St.~Malo, France.
}
\makeatother

\setcounter{footnote}{0}\let\thefootnote\svthefootnote

\section{The problem}
\label{section:introduction}
\vspace{-0.1cm}
\paragraph{Main goal.} The main focus of this work is the notion of total memory capacity ((T)MC), which has been used for decades for measuring the ability of recurrent neural networks (RNNs) to reproduce past inputs. The results in~\cite{RC23} for linear (random) RNNs demonstrate the practical drawbacks of this metric when comparing different linear RNN architectures. This paper shows that these negative conclusions also hold for the MC of nonlinear RNNs. The MC is demonstrated to be not invariant with respect to the scale of the inputs, which can be tuned so that the corresponding MC equals arbitrary values within the bounds known in the literature. This fact has not been reported earlier, posing the question of how valuable some of the existing contributions based on the MC are, and signaling the need for other memory capacity notions that are inherently immune to the reported ambiguity.
\vspace{-0.1cm}
\paragraph{Context and motivation.}
Recurrent neural networks (RNNs) are a popular architecture class for modeling data with temporal dependence or sequential structure and have been widely studied~\cite{Elman1990,LSTM,Pascanu2013}. Different measures of \textit{\textbf{memory capacity}} (MC) for RNNs have been proposed and discussed in the literature~\cite{Haviv2019,Li2021}, which aim at assessing the memory properties of a certain model via measuring its capability to reconstruct past inputs from the model's current states.
\cite{Jaeger:2002} introduced a definition of memory capacity that was originally adapted to a specific class of RNN models with random internal weights, namely, the so-called echo state networks (ESNs)~\cite{Matthews1994,Jaeger04}, which are a popular family of RNNs in the reservoir computing (RC) literature. These RNN architectures can be seen as, in general, nonlinear state-space models with parameters of the state dynamics \textit{randomly sampled} from fixed regular probability distributions and the parameters of the observation equation subject to the training.

In this work, we consider the recurrent neural networks that admit the nonlinear state-space representation given by
\begin{align}
	\mathbf{x}_{t}&= F(\mathbf{x}_{t-1},z_t):=\boldsymbol{\varphi}(A \mathbf{x}_{t-1} + \boldsymbol{C} {z}_{t} + \bm{\xi}) , \label{eq:ESN_state} \\
	\textbf{y}_t & :=h (\mathbf{x}_{t})= W^\top \mathbf{x}_{t} , \label{eq:ESN_readout}
\end{align}
for $t \in \Z_-$, where $F:\mathbb{R}^N \times \mathbb{R}\rightarrow \mathbb{R}^N$ is called the state map, $h: \mathbb{R}^N \rightarrow \mathbb{R}^m$ is the readout map, $\mathbf{x}_t \in \R^N$ are the states, $ {z}_t \in \R$ are the inputs, $ \mathbf{y}_t \in \R^m$ are the outputs, and $m, N \in \N$. The map $\boldsymbol{\varphi} : \Real^N \to \Real^N$ is obtained by componentwise application of an activation function ${\varphi} : \Real \to \Real$. Whenever this map is the identity, \eqref{eq:ESN_state}-\eqref{eq:ESN_readout} correspond to the linear ESN model studied in~\cite{RC23}. In this paper, we assume ${\varphi}$ to be $\mathcal{C}^1$-continuous, bounded, and non-constant. 
The fixed {reservoir (connectivity) matrix} $A\in \mathbb{M}_{N}$, the {input mask} $\boldsymbol{C}\in \mathbb{R}^{N}$, and the {input shift} $\bm{\xi} \in \R^N$ are randomly sampled independently of inputs, while ${W}\in \mathbb{M}_{N,m}$ is estimated via minimizing the empirical risk associated to a chosen loss function.  Additionally, we impose that $A$ is normalized such that $\|A\|_2<1$, with $\|\cdot\|_2$ the spectral norm, which is a sufficient condition for the existence and uniqueness of the solutions of \eqref{eq:ESN_state}-\eqref{eq:ESN_readout} (the so-called Echo State Property (ESP)).

Over the last two decades, MC has received much attention, especially in the RC literature. The idea behind existing classical MC definitions is that the ability of the network states to contain information about previous independent and identically distributed (i.i.d.) inputs can be quantified by the correlation between the outputs of the network and its past inputs and underpins existing contributions for  linear~\cite{Hermans2010,dambre2012,esn2014,linearESN}, echo state shallow~\cite{White2004,farkas:bosak:2016,Verzelli2019a}, and deep architectures~\cite{gallicchioShorttermMemoryDeep2018} (see \cite{dambre2012,GHLO2014,RC3,RC4pv,marzen:capacity,RC15} for recent extensions for temporally dependent inputs). 

Let $(\Omega, \mathcal{A},\mathbb{P})$ be a probability space on which random variables are defined and consider a variance-stationary stochastic input $\mathbf{z}: \Omega \longrightarrow \mathbb{R}^{\mathbb{Z}_-}$ and the associated covariance-stationary output state process $\mathbf{x}: \Omega \longrightarrow\left(\mathbb{R}^N\right)^{Z-}$. Following~\cite{Jaeger:2002,RC15,RC23},  the {\bfi $\bm{\tau}$-lag memory capacity} of the state-space system with respect to $\mathbf{z}$, with $\tau \in \mathbb{N}$, is defined as
\begin{equation}\label{eq:MC_definition}
	\text{MC}_\tau := 
	1 - \frac{1}{\text{Var}({z}_t)} 
	\min_{W \in \mathbb{M}_{N, m}} 
	\mathbb{E} \left[ \,
	({{z}_{t-\tau} - W^\top \mathbf{x}_t})^2 \, 
	\right].
\end{equation}
Under the assumption that the joint process $(\mathbf{x}_t, z_{t-\tau})_{t\in \mathbb{Z}_-}$ is also covariance-stationary and  ${\Gamma}_\mathbf{x} := \text{Cov}(\mathbf{x}_t,\mathbf{x}_t)$ is nonsingular, $\textnormal{MC}_\tau$ has the following closed-form expression (see~\cite{RC15})
\begin{equation}\label{eq:MC_tau_covvar}
	\text{MC}_\tau =
	\frac{\text{Cov}({z}_{t-\tau}, \mathbf{x}_t) {\Gamma}_\mathbf{x}^{-1} \text{Cov}(\mathbf{x}_t, {z}_{t-\tau})}
	{\text{Var}({z}_t)} .
\end{equation}
The {\bfi total memory capacity} (MC) of an ESN is then given by
\begin{equation}\label{eq:MC_total_def}
	\text{MC} := \sum_{\tau = 0}^\infty \text{MC}_\tau
\end{equation}
and a fundamental result~\cite{dambre2012,RC15} is that for i.i.d. inputs and any linear or nonlinear ESN model it necessarily holds that 
\begin{equation}\label{eq:MC_bound}
	1 \leq \textnormal{MC} \leq N. 
\end{equation}
\vspace{-0.1cm}
\paragraph{Contribution.}
We complement the results of~\cite{RC23}, where the memory properties of \textit{linear} ESN models were characterized, and the invariance of MC under specific choices of the linear architecture was proved. In this work, we demonstrate that a direct implication of defining memory as in  \eqref{eq:MC_definition} and \eqref{eq:MC_total_def} is that the memory of a given \textit{nonlinear} ESN model can be tuned to \textit{attain both sides of the memory bounds} \eqref{eq:MC_bound} just by changing the scale of the input.  More specifically, we demonstrate that the change in the variance of the inputs may lead to different capacities. This implies that the MC is not only a function of the RNN architecture and is strictly \textit{task-specific}, its properties also rely on the nature of the input process, and hence, MC serves as a poor measure to discriminate between the signal processing properties of different RNN architectures. 
\vspace{-0.1cm}
\paragraph{Notation.} ${\mathbb{R}^+}$ denotes the set of all nonnegative real numbers.
For an $m$ by $n$ $\mathbb{R}$-valued matrix $A$ we write $\norm{A}_2$ to denote its spectral norm. For any  $m$ by $n$ $\mathbb{R}$-valued matrix $A$ over $\mathbb{R}$, we use the notation $\lfloor A \rfloor$ for its minimal absolute entry that is $
\lfloor A \rfloor := \min_{1\leq i\leq m,1\leq j\leq n} \abs{A_{i j} }
$. The infinity norm of a vector $\mathbf{x} \in \R^N$ is given by $\norm{\mathbf{x}}_\infty = \max_{1 \leq i \leq N} \abs{{x}_i}$.  Let $\mathbf{X}: \Omega \longrightarrow B$ be a random variable with $\left(B,\|\cdot\|_B\right)$ a normed space endowed with a $\sigma$-algebra (for example, its Borel $\sigma$-algebra). We write $\underset{\omega \in \Omega}{\operatorname{ess} \sup }\left\{\|\mathbf{X}(\omega)\|_B\right\}=\inf \left\{b \in {\mathbb{R}^{+}} \mid\|\mathbf{X}\|_B \leq b \enspace {\text{almost}} \enspace {\text{surely}} \right\}$ and $\underset{\omega \in \Omega}{\operatorname{ess} \inf }\left\{\|\mathbf{X}(\omega)\|_B\right\}=\sup \left\{b \in {\mathbb{R}^{+}} \mid\|\mathbf{X}\|_B \geq b \enspace {\text{almost}} \enspace {\text{surely}} \right\}$. 
A discrete-time stochastic process is a map of the type: 
\begin{equation*}
	\label{stochastic process 1}
	\begin{array}{cccc}
		{\bf z}: &\Bbb Z\times \Omega & \longrightarrow &\mathbb{R}^n\\
		&(t, \omega)&\longmapsto &{\bf z}_t(\omega),
	\end{array}
\end{equation*}
such that,  for each $t \in \Bbb Z $, the assignment ${\bf z} _t: \Omega \longrightarrow {\Bbb R}^n  $ is a random variable.
For each $\omega \in  \Omega  $, we will denote by ${\bf z} (\omega):=\{ {\bf z} _t(\omega ) \in \mathbb{R}^n\mid  t \in \Bbb Z  \}$  the {realization} or the {sample path} of the process ${\bf z}$.  We denote $\lceil \mathbf{z}\rceil_{L^{\infty}}:=\underset{\omega \in \Omega}{\operatorname{ess} \sup }\left\{\sup _{t \in \mathbb{Z}} \left\{ \|\mathbf{z}_t(\omega)\|_\infty\right\}\right\}$ and $\lfloor\mathbf{z}\rfloor_{L^{\infty}}:=\underset{\omega \in \Omega}{\operatorname{ess} \inf }\left\{\inf _{t \in \mathbb{Z}} \left\{ \min_{i\in{1,\ldots,n}} |(\mathbf{z}_t(\omega))_i|\right\}\right\}$.

\section{The result}
\label{section:nonlinear}
Our main result shows that the memory capacity of nonlinear ESNs is driven by the relationship between the input mask $\boldsymbol{C}$ and the input process $\{{z}_t\}_{t \in \mathbb{Z}_-}$. The properties of $\{{z}_t\}_{t \in \mathbb{Z}_-}$, and, in particular, its variance, play a key role in our demonstration. We consider the system \eqref{eq:ESN_state}-\eqref{eq:ESN_readout}, where, for simplicity, we set $\bm{\xi} = \mathbf{0}$ and assume a particular type of activation function $\varphi$, which conforms with most of the common choices of RNN architectures. Without loss of generality, we assume that $A$ and $\boldsymbol{C}$ are sampled from some regular distributions and kept fixed. Our goal is to show that for any sampled ESN, one can find appropriate input scalings that make its memory take any value between its theoretical maximum and minimum in \eqref{eq:MC_bound}.
\begin{definition}[\bf Sigmoid activation function]
	We say that an activation function $\varphi: \mathbb{R}\longrightarrow \mathbb{R}$ is a {\it sigmoid} function whenever there exists some $a\in \mathbb{R}$ such that  $|\varphi(x)|\leq a$ ($a$-bounded), for all $x\in \mathbb{R}$, $\varphi \in \mathcal{C}^1(\mathbb{R})$ with $\varphi'(x)\geq 0$ for all $x\in \mathbb{R}$.
\end{definition}
\begin{definition}[\bf Piecewise smooth sigmoid]
	\label{Piecewise smooth sigmoid}
	A sigmoid activation function $\varphi$ is said to be a $\delta$-\textit{linear $D$-saturating sigmoid}, with $\delta, D \in \mathbb{R}^+$, {$\delta< D$}, if the following conditions hold:
	\begin{itemize}
		\item[(i)] $\varphi(x) = x$ for all $x \in (-\delta, \delta)$;
		\item[(ii)] $\varphi(x) = -1$ for all $x \in (-\infty, -D) $ and $\varphi(x) = 1$ for all $x \in (D, +\infty)$.
	\end{itemize}
\end{definition}
This choice of activation function allows us to treat inputs of significantly small or large magnitudes. The conditions (i)-(ii) in Definition~\ref{Piecewise smooth sigmoid}, which specify the linear behavior of $\varphi$ in the $\delta$-neighborhood of $0$ and for inputs of absolute value above some $D$, also simplify the analysis of the situations which are typical for the numerical evaluation of activation functions. More precisely, we notice that due to the finite numerical precision of computers, the computation of $\varphi$ for close to $0$ and large values in the domain renders the behavior compatible with Definition~\ref{Piecewise smooth sigmoid}. 

\paragraph{Attaining the lower MC bound.}
We now prove that such networks, when applied to i.i.d. inputs with sufficiently high variance, attain minimal memory, namely the lower bound in \eqref{eq:MC_bound}.
Consider a process $\{{z}_t\}_{t \in \mathbb{Z}_-}$ with ${z}_t = \sigma \zeta_t$ where $\{\zeta_t\}_{t \in \mathbb{Z}_-}$ is a sequence of i.i.d. Rademacher random variables, which take values $-1$ and $1$ with equal probability, and $\sigma > 0$. We work with dense input mask vectors $\boldsymbol{C}$, which we formalize in the following assumption.
\begin{assumption}\label{assumption:dense_C}
	Define $\underline{c} := \lfloor \boldsymbol{C} \rfloor$ as the minimum absolute entry of input mask  $\boldsymbol{C}$ and let $\underline{c}>0$.
\end{assumption}

\begin{remark}
	In practice, ESN models are often constructed by drawing the entries of the reservoir matrix $A$ from \textit{sparse} distributions. The distribution of the entries of input mask $\boldsymbol{C}$ is rarely chosen to be sparse. Assumption~\ref{assumption:dense_C} immediately holds whenever the entries of $\boldsymbol{C}$ are sampled from any continuous distribution. More generally, $\boldsymbol{C}$ needs to be sampled from the law that has no atom at $0$ or, equivalently, $\min_{1\leq i \leq N}\Prob(C_i = 0) =0$, which is not a restrictive assumption and is satisfied in most practical scenarios. 
\end{remark}

We start by observing that for any $\boldsymbol{C}$ under Assumption~\ref{assumption:dense_C} one can require that $\sigma > {D}/{\underline{c}}$. It is easy to see that in this case
\begin{equation*}
	\lfloor \boldsymbol{C} \mathbf{z}\rfloor_{L^{\infty}}   
	\:=\lfloor \boldsymbol{C} \sigma\,\textnormal{sign}(\boldsymbol{\zeta}) \rfloor_{L^{\infty}}	\:= \: \lfloor \boldsymbol{C}  \rfloor \sigma
	\:\geq\:  \sigma \underline{c}\,  
	\:>\: D.
\end{equation*}
Additionally, one can use Definition~\ref{Piecewise smooth sigmoid} to show that for any $t\in \mathbb{Z}_-$
\begin{equation*}
	\boldsymbol{\varphi}(\boldsymbol{C} {z}_t)
	=
	\boldsymbol{\varphi}\big( \boldsymbol{C} \sigma\,\textnormal{sign}({\zeta}_t) \big) 
	=
	\begin{cases}
		\boldsymbol{\varphi}(- \boldsymbol{C}D/\underline{c})  & \textnormal{if} \quad {z}_t = -\sigma , \\
		\boldsymbol{\varphi}(+ \boldsymbol{C}D/\underline{c})  & \textnormal{if} \quad {z}_t = +\sigma,
	\end{cases} 
\end{equation*}
which leads to the following definition.
\begin{definition}[\bf Extreme states of the ESN]
	\label{Extreme states of ESN} Given an ESN constructed using the piecewise smooth sigmoid activation $\varphi$, the reservoir matrix $A$, and the input mask vector $\boldsymbol{C}$ that satisfies Assumption~\ref{assumption:dense_C}, let $\{{\zeta}_t\}_{t \in \mathbb{Z}_-}$ be a sequence of i.i.d.~random variables with unit variance. Given a scaling  $\sigma> {D}/{\underline{c}}$ and the associated inputs $\{{z}_t\}_{t \in \mathbb{Z}_-}$, ${z}_t=\sigma \zeta_t$, $t\in \mathbb{Z}_-$, for which
	\begin{equation*}
		\boldsymbol{\varphi}(\boldsymbol{C} {z}_t)
		=
		\begin{cases}
			\boldsymbol{\varphi}(- \boldsymbol{C}D/\underline{c})=:\mathbf{x}_{-}  & \textnormal{if} \quad {z}_t = -\sigma , \\
			\boldsymbol{\varphi}(+ \boldsymbol{C}D/\underline{c})=:\mathbf{x}_{+}  & \textnormal{if} \quad {z}_t = +\sigma,
		\end{cases} 
	\end{equation*}
	we call $\mathbf{x}_{+}, \mathbf{x}_{-} \in \{-1,+1\}^N$ \textit{extreme states} of the ESN associated to the input $\{{\zeta}_t\}_{t \in \mathbb{Z}_-}$. 
\end{definition} 
The following result shows that, whenever the ESN state dynamics are defined by~\eqref{eq:ESN_state}-\eqref{eq:ESN_readout}, one can always choose a scaling  $\sigma$ of the input such that the ESN states are all extreme.

\begin{lemma}\label{lemma:lower_bound}
	Let $\mathbf{z}=\{{z}_t\}_{t \in \mathbb{Z}_-}$, $z_t \in \{-1,+1\}$, be a sequence of i.i.d.~rescaled Rademacher random variable of standard deviation $\sigma$. Let $\mathbf{x}=\{\mathbf{x}_t\}_{t \in \mathbb{Z}_-}$, $\mathbf{x}_t\in \mathbb{R}^N$, be the states associated to inputs $\{{z}_t\}_{t \in \mathbb{Z}_-}$ of the ESN with $D$-saturating sigmoid and randomly sampled $A$ and $\boldsymbol{C}$. 
	If $\sigma > \dfrac{\sqrt{N}\|A\|_2 + D}{\underline{c}}$ with $\underline{c} = \lfloor{\boldsymbol{C}}\rfloor$, then for any $t\in \mathbb{Z}_-$
	\begin{equation*}
		\mathbf{x}_t 
		=
		\begin{cases}
			\mathbf{x}_{-}, & \textnormal{if} \quad {z}_t = -\sigma , \\
			\mathbf{x}_{+}, & \textnormal{if} \quad {z}_t = +\sigma.
		\end{cases}
	\end{equation*}
\end{lemma}

\begin{proof}
	By construction, it holds that
	\begin{equation*}
		\lfloor \boldsymbol{C} \mathbf{z} + A \mathbf{x}\rfloor_{L^{\infty}} 		\:\geq\: \lfloor\boldsymbol{C} \mathbf{z}\rfloor_{L^{\infty}}  - \lceil{ A \mathbf{x} }\rceil_{L^\infty}
		\:\geq\: \underline{c}\, \sigma - \sqrt{N}\|A\|_2 .
	\end{equation*}
	Hence, whenever $\sigma > ({\sqrt{N}\|A\|_2 + D)}/{\underline{c}}$, then $\lfloor \boldsymbol{C} \mathbf{z} + A \mathbf{x}\rfloor_{L^{\infty}} 		> D.
	$
	Thus, by Definition~\ref{Piecewise smooth sigmoid} when ${z}_t > 0$, $t\in \mathbb{Z}_-$,  it follows that $\boldsymbol{\varphi}(\boldsymbol{C} {z}_t + A \mathbf{x}_t) = \mathbf{x}_{+}$ for any possible state $\mathbf{x}_t$. The same applies to the case ${z}_t < 0$, $t\in \mathbb{Z}_-$,  with respect to $\mathbf{x}_{-}$. Therefore, whenever $\sigma > ({\sqrt{N}\|A\|_2 + D)}/{\underline{c}}$, the ESN state space consists of the set of extremal states almost surely. 
	\end{proof}
	\begin{remark}
Notice that although using a different bound such that $\lceil{ A \mathbf{x} }\rceil_{L^\infty}\leq \overline{a}N$ with $\overline{a}=\|A\|_\infty$ one obtains a less sharp larger lower bound for $\sigma$, it yields more stable empirical illustrations.
\end{remark}
\vspace{-0.1cm}
\paragraph{Attaining the upper MC bound.}
\label{Upper Bound Saturation}
We now show that the states of ESNs that are constructed using the i.i.d.~inputs with sufficiently small variance are equivalent to those of a linear recurrent network (or linear ESN), which yields provably maximal memory. 
We use again the i.i.d.~rescaled Rademacher sequence $\{{z}_t\}_{t \in \mathbb{Z}}$ of standard deviation $\sigma$ and notice that, for $\overline{c} := \norm{\boldsymbol{C}}_\infty$, it holds
\begin{equation*}
\lceil\boldsymbol{C} \mathbf{z}\rceil_{L^{\infty}} \leq \norm{\boldsymbol{C}}_\infty \lceil \mathbf{z}\rceil_{L^{\infty}} = \overline{c}\, \sigma.
\end{equation*}

\begin{lemma}\label{lemma:upper_bound}
Let $\{{z}_t\}_{t \in \mathbb{Z}}$ be a sequence of i.i.d. rescaled Rademacher random variables of standard deviation $\sigma$. Let $\{\mathbf{x}_t\}_{t \in \mathbb{Z}}$ be the states associated with inputs $\{{z}_t\}_{t \in \mathbb{Z}}$ of the nonlinear  ESN with $\delta$-linear sigmoid, $A$ and $\boldsymbol{C}$ drawn from random distributions such that $\|A\|_2<1$.
If $\sigma < \frac{\delta(1-\|A\|_2)}{\overline{c}}$, with $\overline{c}=\|\boldsymbol{C}\|_\infty$, then $\{\mathbf{x}_t\}_{t \in \mathbb{Z}}$ are equivalent to the states of the linear ESN model with the state dynamics defined as
\begin{equation*}
	{\mathbf{x}}_t := {A} {\mathbf{x}}_{t-1} + \boldsymbol{C} {z}_t .
\end{equation*}
\end{lemma}

\begin{proof}
We begin by bounding the maximal possible entry of the states of the ESN 
We write that
\begin{align*}
	\lceil \mathbf{x}\rceil_{L^{\infty}}
	&:=\underset{\omega \in \Omega}{\operatorname{ess} \sup }\left\{\sup _{t \in \mathbb{Z}_-}\left\{\left\|\mathbf{x}_t(\omega)\right\|_\infty\right\}\right\}
	=\underset{\omega \in \Omega}{\operatorname{ess} \sup }\left\{\sup _{t \in \mathbb{Z}_-}\left\{
	\left\| \sum_{j=0}^\infty A^j \boldsymbol{C} z_{t-j}(\omega)\right\|_\infty
	\right\}\right\} \\
	& \leq \underset{\omega \in \Omega}{\operatorname{ess} \sup }\left\{\sup _{t \in \mathbb{Z}_-}\left\{
	\sum_{j=0}^\infty \left\| A^j \boldsymbol{C} z_{t-j}(\omega)\right\|_\infty
	\right\}\right\} \\
	& \leq \underset{\omega \in \Omega}{\operatorname{ess} \sup }\left\{\sup _{t \in \mathbb{Z}_-}\left\{
	\sum_{j=0}^\infty \left\| A \right\|_2^j \cdot \left\|\boldsymbol{C} z_{t-j}(\omega)\right\|_2
	\right\}\right\} \\
	& \leq \underset{\omega \in \Omega}{\operatorname{ess} \sup }\left\{\sup _{t \in \mathbb{Z}_-}\left\{
	\sqrt{N} \sum_{j=0}^\infty \left\| A \right\|_2^j \cdot \left\|\boldsymbol{C} z_{t-j}(\omega)\right\|_\infty
	\right\}\right\} \\
	& \leq \underset{\omega \in \Omega}{\operatorname{ess} \sup }\left\{\sup _{t \in \mathbb{Z}_-}\left\{
	\sqrt{N} \sum_{j=0}^\infty \left\| A \right\|_2^j \cdot \overline{c}\sigma
	\right\}\right\} 
	\leq \dfrac{\sqrt{N}}{1-\|A\|_2} \overline{c}\sigma,
\end{align*}
where we used that $\|A\|_2<1$. Finally, notice that whenever $\sigma < \delta(1-\|A\|_2)/(\sqrt{N}\overline{c})$, then necessarily $\lceil \mathbf{x}\rceil_{L^{\infty}}<\delta$ which yields that for the $\delta$-linear sigmoid $\boldsymbol{\varphi}(A\mathbf{x}_t + \boldsymbol{C} z_t) = A\mathbf{x}_t + \boldsymbol{C} z_t$ for all $t\in \mathbb{Z}_-$ as required.
\end{proof}

\begin{remark}
Note that the proof of Lemma~\ref{lemma:upper_bound} can be generalized to work for any input sequence that is almost surely bounded. For the sake of simplicity and to be consistent with our proof of Lemma~\ref{lemma:lower_bound}, here we have limited ourselves to the case of rescaled Rademacher processes.
We also note that, in simulations, a looser bound, i.e. one where the $\sqrt{N}$ factor is omitted and thus $\sigma < \delta(1-\|A\|_2)/\overline{c}$, still yields robust empirical results.
\end{remark}
\vspace{-0.3cm}
\paragraph{Total Memory Saturation.}
\label{Total Memory Saturation}
We are now ready to state our main theoretical result regarding the total memory capacity of nonlinear ESN models.

\begin{theorem}
\label{theorem:MC_sharp} 
Let $\{{z}_t\}_{t \in \mathbb{Z}}$ be a sequence of i.i.d. rescaled Rademacher random variables of standard deviation $\sigma$. Let a ESN model be defined as in \eqref{eq:ESN_state}-\eqref{eq:ESN_readout}. Let ${\varphi}$ be a  $\delta$-linear $D$-saturating sigmoid activation, and let $A$ and $\boldsymbol{C}$ be drawn from random distributions. Assume that $\|A\|_2<1$ and that Assumption~\ref{assumption:dense_C} holds.  Let $\{\mathbf{x}_t\}_{t \in \mathbb{Z}}$ be the states associated with inputs $\{{z}_t\}_{t \in \mathbb{Z}}$ of the   ESN.	Then, there exist positive constants $\overline{\sigma}_{A,C}$ and $\underline{\sigma}_{A,C}$ such that
\begin{description}
	\item[(i)] If $\sigma > \overline{\sigma}_{A,C}$, then $\textnormal{MC} = 1$, almost surely.
	\item[(ii)] If $\sigma < \underline{\sigma}_{A,C}$, then $\textnormal{MC} = N$, almost surely.
\end{description}
\end{theorem}
\begin{proof}
{\bf{(i)}}
We use Lemma~\ref{lemma:lower_bound} to show that $\textnormal{MC}_0 = 1$, while $\textnormal{MC}_\tau = 0$ for all $\tau \geq 1$ whenever 
$\sigma > \overline{\sigma}_{A,C}$ with $\overline{\sigma}_{A,C}:=\dfrac{\sqrt{N}\|A\|_2 + D}{\underline{c}}$, where $\underline{c}=\lfloor{\boldsymbol{C}}\rfloor$.
Without loss of generality, for $\tau = 0$ set $W_0 = \sigma \boldsymbol{e}_1$ if the first entry of $\mathbf{x}_{+}$ is $1$, and $W_0 = - \sigma \boldsymbol{e}_1$ if $(\mathbf{x}_{+})_1=-1$, with $\boldsymbol{e}_1\in \mathbb{R}^N$ the first canonical basis vector. By construction and Lemma~\ref{lemma:lower_bound}, $W_0^\top \mathbf{x}_{+} = \sigma$ and $W_0^\top \mathbf{x}_{-} = -\sigma$, and hence
\begin{equation*}
	\min_{W_0 \in \mathbb{R}^{N}} 
	\mathbb{E} 
	\left[ \big( 
	{z}_{t} - W_0^\top \mathbf{x}_t
	\big)^2 \right]
	= 0 .
\end{equation*}
When $\tau > 1$, Lemma~\ref{lemma:lower_bound} implies that ${z}_{t-\tau}$ and $\mathbf{x}_t = \boldsymbol{\varphi}(\boldsymbol{C} {z}_t)$ are independent random variables for any $t\in \mathbb{Z}_{-}$. Therefore, since $\mathbb{E}\big[ {z}_{t-\tau} \big] = 0$,
\begin{equation*}
	\min_{W_\tau \in \mathbb{R}^{N}} 
	\mathbb{E} 
	\left[ \big( 
	{z}_{t-\tau} - W_\tau^\top \mathbf{x}_t
	\big)^2 \right]
	=
	\mathbb{E}\big[ {z}_{t-\tau}^2 \big]
	+ \min_{W_\tau \in \mathbb{R}^{ N}} \mathbb{E}\big[ (W_\tau^\top \mathbf{x}_t)^2 \big]
	= 
	\sigma^2 ,
\end{equation*}
and hence $\textnormal{MC}_{\tau} = 0$ for all $\tau\geq 1$. This shows that $\textnormal{MC} = 1$ almost surely.

\noindent {\bf{(ii)}}
Let  now 
\begin{equation*}
	\underline{\sigma}_{A,C} = \frac{\delta(1-\|A\|_2)}{\sqrt{N}\overline{c}},
\end{equation*} 
where $\overline{c}=\|\boldsymbol{C}\|_\infty$.   
Then, by Lemma~\ref{lemma:upper_bound},
for $\sigma < \underline{\sigma}_{A,C}$ the nonlinear ESN with ${\varphi}$ a  $\delta$-linear $D$-saturating sigmoid activation is equivalent to a linear ESN model. As proven in~\cite{RC15}, linear ESNs have maximal memory almost surely, therefore $\textnormal{MC} = N$ almost surely.
\end{proof}

\section{Simulations}
To give an empirical illustration of Theorem~\ref{theorem:MC_sharp}, we provide in Figure~\ref{fig:compare} the results of a simulation aimed at computing the total memory capacity of different nonlinear ESN models with i.i.d. rescaled Rademacher inputs, dense input matrix and tanh activation function.\footnote{The simulation code is publicly available at \url{https://github.com/Learning-of-Dynamic-Processes/memorycapacity}.}
These clearly show that, while for some types of reservoir matrices it is indeed possible to achieve practically maximal memory, as the variance of the input process increases, all models eventually drop to the lower bound of $\text{MC} = 1$ (see~\cite{RC23} for an in-depth discussion of why, even in the linear ESN limit, it is generally not possible to obtain precise memory estimates by simulation due to intrinsic numerical issues).

\begin{figure}[t!]
\centering
$\varphi = \text{tanh}$\\[5pt]
\begin{subfigure}[b]{0.32\textwidth}
	\centering
	\includegraphics[width=\textwidth]{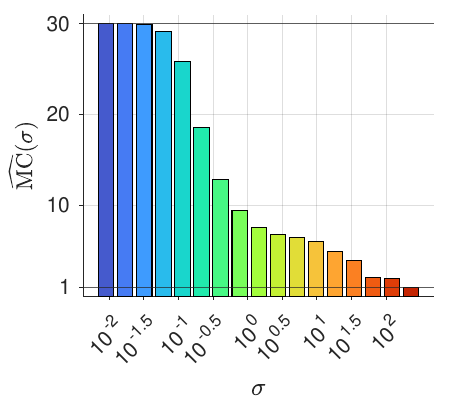}
	\caption{$A_{ij} \sim\ \mathcal{O}(\mathcal{N}(0,1))$}
	\label{}
\end{subfigure}
\hfill
\begin{subfigure}[b]{0.32\textwidth}
	\centering
	\includegraphics[width=\textwidth]{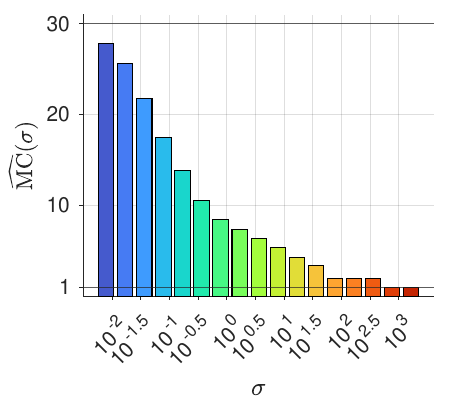}
	\caption{$A_{ij} \sim\ sp_C\mathcal{N}(0,1,0.1,0.7)$}
	\label{}
\end{subfigure}
\hfill
\begin{subfigure}[b]{0.32\textwidth}
	\centering
	\includegraphics[width=\textwidth]{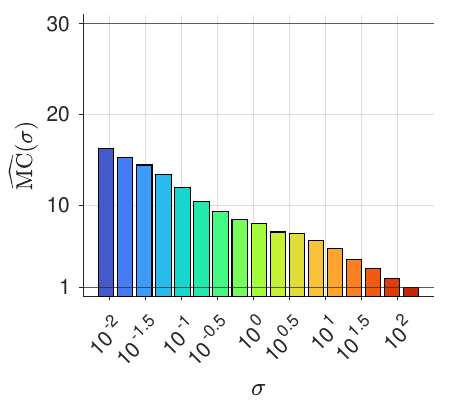}
	\caption{$A_{ij} \sim\ \mathcal{N}(0,1)$}
	\label{}
\end{subfigure}
\caption{Plot of estimated total memory capacity of nonlinear ESN models, $\widehat{\textnormal{MC}}({\sigma})$, as a function of scaled Rademacher standard deviation, $\sigma \in [\underline{\sigma}_{A,C}, \overline{\sigma}_{A,C}]$. The horizontal axis shows the value of $\sigma$ in logarithmic scale. Connectivity matrix $A = (A_{ij}) \in \mathbb{M}_N$ -- with $N = 30$ and $\|A\|_2 = 0.95$ -- sampled from (a) orthogonal ensemble (via Gram-Schmidt decomposition of entry-wise standard Gaussian matrix); (b) conditioning sparse Gaussian ensemble with sparsity degree $0.1$ and conditioning number $0.7$; (c) standard Gaussian ensemble. In all panels $c_i \sim\ \text{i.i.d.}\ \mathcal{N}(0,1)$. For each value of $\sigma$, we use $10^5$ Monte Carlo replications to estimate $\widehat{\textnormal{MC}}({\sigma})$.}
\label{fig:compare}
\end{figure}

Additionally, we empirically study the memory properties of nonlinear ESNs that are not explicitly within the scope of Theorem~\ref{theorem:MC_sharp}. We consider two alternative activation functions that cannot be well approximated by piecewise smooth sigmoid maps in Definition~\ref{Piecewise smooth sigmoid}.
First, we simulate ESN models with a ReLU activation, $\text{ReLU}(x) = \max(x, 0)$. Second, we apply a sigmoid-like map that is approximately linear around $0$ with logarithmic growth at $\pm \infty$, given by
\begin{equation*}
\text{LogSig}(x) := \text{sign}(x) \log(1 + \abs{x}) .
\end{equation*} 
Figure~\ref{fig:other} shows the results obtained with these two alternative activations within the same simulation design used for Figure~\ref{fig:compare}. One can immediately notice that a ReLU activation, which is piecewise-linear, significantly reduces total memory capacity at all values of $\sigma$. Moreover, the particular law to sample $A$ does not appear to impact these results.
On the other hand, we find that the LogSig activation, which does not saturate even for large inputs, yields improved memory performance for high-input variance. However, this comes at the cost of faster decay in memory in the pseudo-linear regime when $\sigma$ is small.
\vspace{-0.2cm}
\paragraph{Discussion.}
Our simulations suggest that, while activation functions described by Definition~\ref{Piecewise smooth sigmoid} have poor memory performance at high input variance and the ReLU map strongly caps total memory capacity over the entire variance range, LogSig activation may help mitigate memory issues. 
This, however, is not the case. In fact, it is easy to check that LogSig ESN states still become extreme for large $\sigma$ (we do not report these simulations due to space constraints; they can be reproduced using our code by setting e.g. $\sigma = 10^8$). 
Hence, our analysis of state saturation also holds more broadly for LogSig activation, although for a larger range of $\sigma$ values.
\begin{figure}[t!]
\centering
$\varphi = \text{ReLU}$\\[5pt]
\begin{subfigure}[b]{0.32\textwidth}
	\centering
	\includegraphics[width=\textwidth]{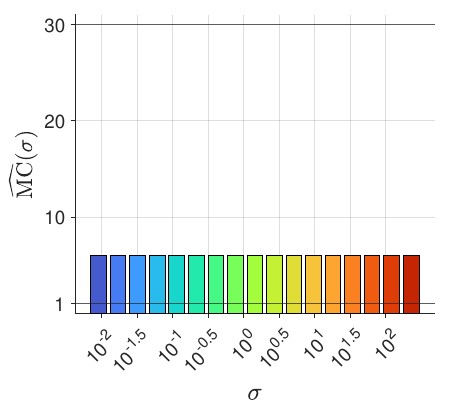}
	\caption{$A_{ij} \sim\ \mathcal{O}(\mathcal{N}(0,1))$}
	\label{}
\end{subfigure}
\hfill
\begin{subfigure}[b]{0.32\textwidth}
	\centering
	\includegraphics[width=\textwidth]{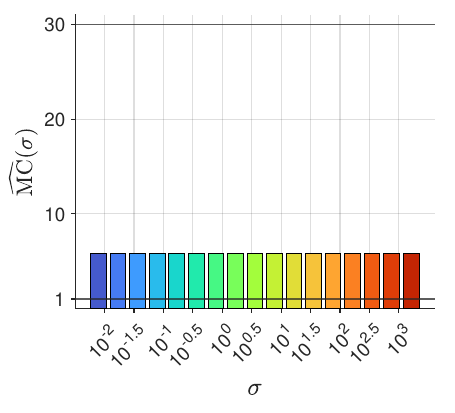}
	\caption{$A_{ij} \sim\ sp_C\mathcal{N}(0,1,0.1,0.7)$}
	\label{}
\end{subfigure}
\hfill
\begin{subfigure}[b]{0.32\textwidth}
	\centering
	\includegraphics[width=\textwidth]{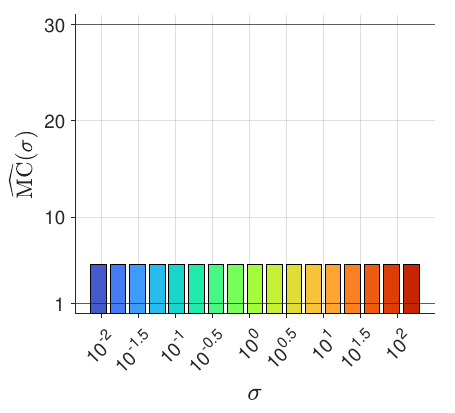}
	\caption{$A_{ij} \sim\ \mathcal{N}(0,1)$}
	\label{}
\end{subfigure}
\\
\rule{\textwidth}{0.2pt}
\\[5pt]
$\varphi = \text{LogSig}$\\[5pt]
\begin{subfigure}[b]{0.32\textwidth}
	\centering
	\includegraphics[width=\textwidth]{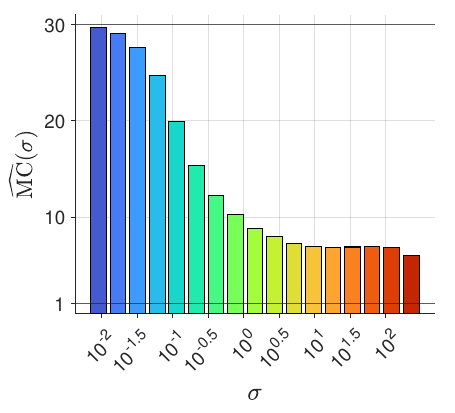}
	\caption{$A_{ij} \sim\ \mathcal{O}(\mathcal{N}(0,1))$}
	\label{}
\end{subfigure}
\hfill
\begin{subfigure}[b]{0.32\textwidth}
	\centering
	\includegraphics[width=\textwidth]{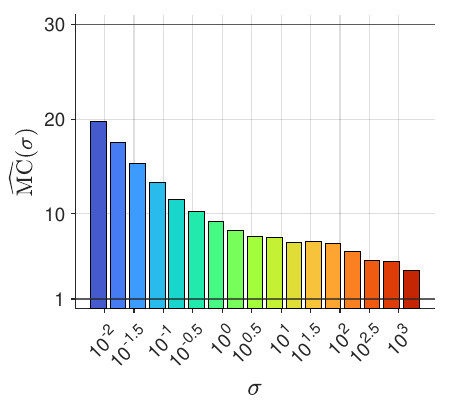}
	\caption{$A_{ij} \sim\ sp_C\mathcal{N}(0,1,0.1,0.7)$}
	\label{}
\end{subfigure}
\hfill
\begin{subfigure}[b]{0.32\textwidth}
	\centering
	\includegraphics[width=\textwidth]{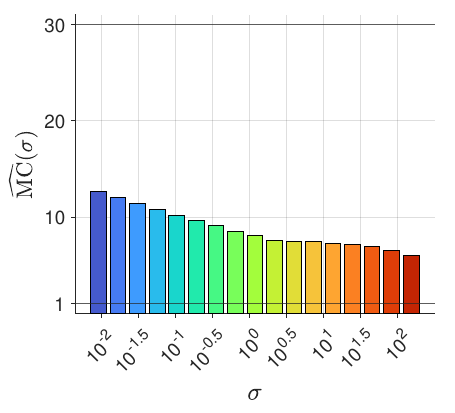}
	\caption{$A_{ij} \sim\ \mathcal{N}(0,1)$}
	\label{}
\end{subfigure}
\caption{Plot of estimated total memory capacity of nonlinear ESN models, $\widehat{\textnormal{MC}}({\sigma})$, as a function of scaled Rademacher standard deviation, $\sigma \in [\underline{\sigma}_{A,C}, \overline{\sigma}_{A,C}]$. The horizontal axis shows the value of $\sigma$ in logarithmic scale. The simulation design is identical to that used to produce Figure~\ref{fig:compare}.}
\label{fig:other}
\end{figure}

\section{Conclusion}
\label{section:conclusion}

In this work, we have shown that without making any precise assumptions on the properties of the inputs of a state space system other than their i.i.d. character, it is \textit{a priori} impossible to establish a nonlinear memory characterization that is more precise than just its boundedness within the range $1 \leq \textnormal{MC} \leq N $.  The definition of MC due to~\cite{Jaeger:2002}, which is often studied in the literature on reservoir computing methods, depends on the specific nature of the data. To be more precise, the variance of inputs was shown to directly control whether any fixed nonlinear ESN has maximal or minimal total memory capacity, based on whether it behaves approximately like a linear model or, rather, its state space consists only of two extremal states.

These results, taken together with those developed in~\cite{RC23} for linear recurrent networks, effectively spell out the limits of the MC concept.
Our conclusion is that memory capacity should not be considered a useful metric to evaluate the inherent ``qualities'' or ``capabilities'' of an ESN or RNN model, as it is merely a statistical quantity that may offer some insight into the application of a given model to a given input sequence.
\vspace{-0.1cm}
\paragraph{Potential geometric analysis extensions.} This work can have interesting extensions with the help of geometric tools. We emphasize two important cases for which such analysis could be especially appropriate: (i) taking into account geometric properties of the input-output system that needs to be modeled by ESNs, and (ii) geometric features of the network itself. In many applications, accurately capturing the intrinsic structure of data requires modeling it as elements of a non-trivial manifold, rather than just as a Euclidean space. The same applies to those cases when the states of the network exhibit similar behavior and are defined on manifolds. This perspective is crucial in domains like neuroimaging and computational anatomy~\cite{Vaillant2004}, natural language processing~\cite{nickel2017}, and robotics (where configuration spaces are often non-Euclidean~\cite{Park2018}), and others. By leveraging tools from differential geometry and topological data analysis, one can enhance representation learning in neural networks and propose a geometry-informed concept of memory capacity that avoids the limitations discussed in this paper. This direction has already been initiated in some earlier work (see, for example,~\cite{tino:symmetric}). Additionally, exploring the correlation between memory capacity and topological invariants (for example, Betti numbers) that persist across time delays could be an interesting direction. Investigating how the geometry of the input-output system or the input signal itself affects memory is also an avenue for our future research. More broadly, the extension of the circle of ideas in this paper to a non-Euclidean context is of much relevance in the design of structure-preserving machine learning algorithms that yield dynamical proxies of systems with underlying variational or geometric structures despite the presence of approximation and/or estimation error. Even though progress has been made in this direction using recurrent networks (see~\cite{RCSP1} and references therein), most work in this field has been carried out using kernel methods on manifolds (see~\cite{da2023gaussian,dacosta2024geometriclearningpositivelydecomposable,dacosta2023invariantkernelsriemanniansymmetric,RCSP2,RCSP3} and references therein).

\subsubsection*{Acknowledgements}
GB thanks the Center for Doctoral Studies in Economics and the University of Mannheim, where a significant part of this work was developed during the course of doctoral studies. GB also thanks the University of St.~Gallen for the hospitality over the years. LG thanks the hospitality of Nanyang Technological University, where part of this work was completed. JPO thanks the hospitality of the University of St.~Gallen, where part of this work was completed and acknowledges partial financial support from the School of Physical and Mathematical Sciences of the Nanyang Technological University, Singapore.

\noindent
\addcontentsline{toc}{section}{Bibliography}
\bibliographystyle{wmaainf}
\bibliography{RC23_bib}

\begin{thebibliography}{Damb~12}

\bibitem[Ball~24]{RC23}
G.~Ballarin, L.~Grigoryeva, and J.-P. Ortega.
\newblock ``{Memory of recurrent networks: Do we compute it right?}''.
\newblock {\it Journal of Machine Learning Research}, Vol.~25, No.~243,
  pp.~1--38, 2024.

\bibitem[Bara~14]{esn2014}
P.~Barancok and I.~Farkas.
\newblock ``{Memory capacity of input-driven echo state networks at the edge of
  chaos}''.
\newblock In: {\it Proceedings of the International Conference on Artificial
  Neural Networks (ICANN)}, pp.~41--48, 2014.

\bibitem[Cost~24]{dacosta2024geometriclearningpositivelydecomposable}
N.~D. Costa, C.~Mostajeran, J.-P. Ortega, and S.~Said.
\newblock ``{Geometric Learning with Positively Decomposable Kernels}''.
\newblock {\it Journal of Machine Learning Reasearch}, Vol.~25, No.~326,
  pp.~1--42, 2024.

\bibitem[Cost~25]{dacosta2023invariantkernelsriemanniansymmetric}
N.~D. Costa, C.~Mostajeran, J.-P. Ortega, and S.~Said.
\newblock ``{Invariant kernels on Riemannian symmetric spaces: a
  harmonic-analytic approach}''.
\newblock {\it SIAM Journal on Mathematics of Data Science}, Vol.~7, No.~2,
  pp.~752--776, 2025.

\bibitem[Coui~16]{linearESN}
R.~Couillet, G.~Wainrib, H.~Sevi, and H.~T. Ali.
\newblock ``{The asymptotic performance of linear echo state neural
  networks}''.
\newblock {\it Journal of Machine Learning Research}, Vol.~17, No.~178,
  pp.~1--35, 2016.

\bibitem[Da C~23]{da2023gaussian}
N.~{Da Costa}, C.~Mostajeran, and J.-P. Ortega.
\newblock ``{The Gaussian kernel on the circle and spaces that admit isometric
  embeddings of the circle}''.
\newblock In: {\it International Conference on Geometric Science of
  Information}, pp.~426--435, Springer, 2023.

\bibitem[Damb~12]{dambre2012}
J.~Dambre, D.~Verstraeten, B.~Schrauwen, and S.~Massar.
\newblock ``{Information processing capacity of dynamical systems}''.
\newblock {\it Scientific reports}, Vol.~2, No.~514, 2012.

\bibitem[Elma~90]{Elman1990}
J.~E. Elman.
\newblock ``{Finding structure in time}''.
\newblock {\it Cognitive Science}, Vol.~14, No.~2, pp.~179--211, 1990.

\bibitem[Fark~16]{farkas:bosak:2016}
I.~Farkas, R.~Bosak, and P.~Gergel.
\newblock ``{Computational analysis of memory capacity in echo state
  networks}''.
\newblock {\it Neural Networks}, Vol.~83, pp.~109--120, 2016.

\bibitem[Gall~18]{gallicchioShorttermMemoryDeep2018}
C.~Gallicchio.
\newblock ``{Short-term memory of Deep RNN}''.
\newblock 2018.

\bibitem[Gono~20]{RC15}
L.~Gonon, L.~Grigoryeva, and J.-P. Ortega.
\newblock ``{Memory and forecasting capacities of nonlinear recurrent
  networks}''.
\newblock {\it Physica D}, Vol.~414, No.~132721, pp.~1--13., 2020.

\bibitem[Grig~14]{GHLO2014}
L.~Grigoryeva, J.~Henriques, L.~Larger, and J.-P. Ortega.
\newblock ``{Stochastic time series forecasting using time-delay reservoir
  computers: performance and universality}''.
\newblock {\it Neural Networks}, Vol.~55, pp.~59--71, 2014.

\bibitem[Grig~16a]{RC3}
L.~Grigoryeva, J.~Henriques, L.~Larger, and J.-P. Ortega.
\newblock ``{Nonlinear memory capacity of parallel time-delay reservoir
  computers in the processing of multidimensional signals}''.
\newblock {\it Neural Computation}, Vol.~28, pp.~1411--1451, 2016.

\bibitem[Grig~16b]{RC4pv}
L.~Grigoryeva, J.~Henriques, and J.-P. Ortega.
\newblock ``{Reservoir computing: information processing of stationary
  signals}''.
\newblock In: {\it Proceedings of the 19th IEEE International Conference on
  Computational Science and Engineering}, pp.~496--503, 2016.

\bibitem[Havi~19]{Haviv2019}
D.~Haviv, A.~Rivkind, and O.~Barak.
\newblock ``{Understanding and controlling memory in recurrent neural
  networks}''.
\newblock In: {\it Proceedings of the 36th International Conference on Machine
  Learning}, pp.~2663--2671, 2019.

\bibitem[Herm~10]{Hermans2010}
M.~Hermans and B.~Schrauwen.
\newblock ``{Memory in linear recurrent neural networks in continuous time.}''.
\newblock {\it Neural Networks}, Vol.~23, No.~3, pp.~341--55, apr 2010.

\bibitem[Hoch~97]{LSTM}
S.~Hochreiter and J.~Schmidhuber.
\newblock ``{Long short-term memory}''.
\newblock {\it Neural Computation}, Vol.~9, No.~8, pp.~1735--1780, 1997.

\bibitem[Hu~25a]{RCSP3}
J.~Hu, J.-P. Ortega, and D.~Yin.
\newblock ``{A global structure-preserving kernel method for the learning of
  Poisson systems}''.
\newblock {\it Journal of Nonlinear Science}, Vol.~35, No.~79, 2025.

\bibitem[Hu~25b]{RCSP2}
J.~Hu, J.-P. Ortega, and D.~Yin.
\newblock ``{A structure-preserving kernel method for learning Hamiltonian
  systems}''.
\newblock {\it To appear in Mathematics of Computation}, 2025.

\bibitem[Jaeg~02]{Jaeger:2002}
H.~Jaeger.
\newblock ``{Short term memory in echo state networks}''.
\newblock {\it Fraunhofer Institute for Autonomous Intelligent Systems.
  Technical Report.}, Vol.~152, 2002.

\bibitem[Jaeg~04]{Jaeger04}
H.~Jaeger and H.~Haas.
\newblock ``{Harnessing nonlinearity: Predicting chaotic systems and saving
  energy in wireless communication}''.
\newblock {\it Science}, Vol.~304, No.~5667, pp.~78--80, 2004.

\bibitem[Li~21]{Li2021}
Z.~Li, J.~Han, W.~E, and Q.~Li.
\newblock ``{On the curse of memory in recurrent neural networks: Approximation
  and optimization analysis}''.
\newblock In: {\it ICLR}, pp.~1--43, 2021.

\bibitem[Marz~17]{marzen:capacity}
S.~Marzen.
\newblock ``{Difference between memory and prediction in linear recurrent
  networks}''.
\newblock {\it Physical Review E}, Vol.~96, No.~3, pp.~1--7, 2017.

\bibitem[Matt~94]{Matthews1994}
M.~Matthews and G.~Moschytz.
\newblock ``{The identification of nonlinear discrete-time fading-memory
  systems using neural network models}''.
\newblock {\it IEEE Transactions on Circuits and Systems II: Analog and Digital
  Signal Processing}, Vol.~41, No.~11, pp.~740--751, 1994.

\bibitem[Nick~17]{nickel2017}
M.~Nickel and D.~Kiela.
\newblock ``{Poincar{\'{e}} embeddings for learning hierarchical
  representations}''.
\newblock In: {\it Proceedings of the 31st International Conference on Neural
  Information Processing Systems}, pp.~6341--6350, Curran Associates Inc., Red
  Hook, NY, USA, 2017.

\bibitem[Orte~24]{RCSP1}
J.-P. Ortega and D.~Yin.
\newblock ``{Learnability of linear port-Hamiltonian systems}''.
\newblock {\it Journal of Machine Learning Research}, Vol.~25, pp.~1--56, 2024.

\bibitem[Park~18]{Park2018}
F.~C. Park, B.~Kim, C.~Jang, and J.~Hong.
\newblock ``{Geometric Algorithms for Robot Dynamics: A Tutorial Review}''.
\newblock {\it Applied Mechanics Reviews}, Vol.~70, No.~1, feb 2018.

\bibitem[Pasc~13]{Pascanu2013}
R.~Pascanu, C.~Gulcehre, K.~Cho, and Y.~Bengio.
\newblock ``{How to construct deep recurrent neural networks}''.
\newblock {\it arXiv}, dec 2013.

\bibitem[Tino~18]{tino:symmetric}
P.~Tino.
\newblock ``{Asymptotic Fisher memory of randomized linear symmetric echo state
  networks}''.
\newblock {\it Neurocomputing}, Vol.~298, pp.~4--8, 2018.

\bibitem[Vail~04]{Vaillant2004}
M.~Vaillant, M.~I. Miller, L.~Younes, and A.~Trouv{\'{e}}.
\newblock ``{Statistics on diffeomorphisms via tangent space
  representations.}''.
\newblock {\it NeuroImage}, Vol.~23 Suppl 1, pp.~S161--9, 2004.

\bibitem[Verz~19]{Verzelli2019a}
P.~Verzelli, C.~Alippi, and L.~Livi.
\newblock ``{Echo State Networks with self-normalizing activations on the
  hyper-sphere}''.
\newblock {\it Scientific Reports}, Vol.~9, No.~1, p.~13887, dec 2019.

\bibitem[Whit~04]{White2004}
O.~White, D.~Lee, and H.~Sompolinsky.
\newblock ``{Short-term memory in orthogonal neural networks}''.
\newblock {\it Physical Review Letters}, Vol.~92, No.~14, p.~148102, apr 2004.

\end{thebibliography}

\end{document}